\documentclass{article}

\usepackage{arxiv}

\usepackage[utf8]{inputenc} 
\usepackage[T1]{fontenc}    
\usepackage{hyperref}       
\usepackage{url}            
\usepackage{booktabs}       
\usepackage{amsfonts}       
\usepackage{nicefrac}       
\usepackage{microtype}      
\usepackage{lipsum}
\usepackage{graphicx}
\graphicspath{ {./images/} }

\usepackage[utf8]{inputenc} 
\usepackage[T1]{fontenc}    
\usepackage{hyperref}       
\usepackage{url}            
\usepackage{booktabs}       
\usepackage{amsfonts}       
\usepackage{nicefrac}       
\usepackage{microtype}      
\usepackage{xcolor}         

\usepackage{graphicx}
\usepackage{algorithm}
\usepackage{amsmath}
\usepackage{amssymb}
\usepackage{float}
\usepackage{pifont}
\usepackage{algpseudocode}
\usepackage{natbib}
\usepackage{hyperref}
\usepackage{multirow}
\usepackage{amsthm}
\newtheorem{proposition}{Proposition}
\usepackage{textcomp}

\title{CHASM: Cross-frequency Harmonized
Axis-Separable Mixing for Spectral Token Operators}

\author{
Pengcheng Fang\textsuperscript{1},
Hongli Chen\textsuperscript{2},
Yuxia Chen\textsuperscript{3},
Tengjiao Sun\textsuperscript{1},
Jiaxin Liu\textsuperscript{4},
Xiaohao Cai\textsuperscript{1, \textdagger}
\\[0.8em]
\textsuperscript{1}University of Southampton \quad
\textsuperscript{2}University of Queensland \quad \\
\textsuperscript{3}Chengdu University of Technology \quad
\textsuperscript{4}University of Illinois Urbana-Champaign
\\[0.4em]
\textsuperscript{\textdagger}Corresponding author.
}

\begin{document}

\maketitle

\begin{abstract}
Spectral token mixers based on Fourier transforms provide an efficient way to
model global interactions in visual feature maps. Existing designs often either
apply filter-wise spectral responses along fixed channel axes, or learn adaptive
frequency-indexed channel mixing without explicitly aligning the channel
directions used across frequencies. We propose CHASM, a Cross-frequency
Harmonized Axis-Separable Mixer, as a structured middle ground. CHASM separates
what should be shared from what should remain frequency-specific: all frequencies
share a learned channel eigenbasis, while each frequency retains its own positive
spectral gains.
The shared basis makes channel directions comparable across the spectrum, whereas the
positive gains preserve local spectral adaptivity. CHASM applies this structured
operator separably along the height and width axes and is used as a drop-in
replacement mixer inside existing backbones. We provide a structural
characterization of the shared-basis operator family and evaluate CHASM through
controlled same-backbone comparisons. Across accelerated MRI reconstruction,
undersampled MRI segmentation, and natural-image reconstruction, CHASM
consistently improves over same-backbone spectral-mixer baselines.
Ablations show that removing the shared-basis constraint weakens performance,
and randomizing coherent sampling geometry substantially reduces the gain,
supporting cross-frequency harmonization as a useful inductive bias for spectral
token operators.
\end{abstract}

\section{Introduction}
\label{sec:introduction}

Spectral token mixers have become a useful mechanism for modeling long-range
interactions in visual feature maps. By transforming features into a frequency
domain, applying learnable operations, and mapping the result back to the
spatial domain, they expose each feature location to global context with a
compact operator~\citep{li2021fourier,guibas2022adaptive,rao2021global,chi2020fast}.
Such designs are especially attractive in inverse imaging and medical-image
analysis, where undersampling, aliasing, and recovery often have explicit
spectral structure~\citep{lustig2007sparse,zbontar2018fastmri}.

Most spectral mixers learn frequency-indexed responses: after a spatial Fourier
transform, a frequency-dependent operation is applied at each retained Fourier
bin. Existing designs often emphasize one of two forms of flexibility.
Filter-wise methods give different frequencies different responses but largely
operate along fixed channel axes. More adaptive spectral mixers allow
frequency-indexed channel transformations, but then each frequency may learn its
own channel directions without an explicit mechanism for aligning those
directions across the spectrum. What is missing is a structured middle ground:
a mixer should learn useful channel directions, make them comparable across
frequencies, and still allow each frequency to respond differently.

We propose CHASM, a Cross-frequency Harmonized Axis-Separable Mixer, to realize
this middle ground. CHASM parameterizes each frequency-indexed channel operator
with a shared learned eigenbasis and a frequency-specific positive gain vector.
The shared basis defines a common channel coordinate system across the spectrum,
whereas the gains preserve frequency-specific spectral adaptivity. In other
words, CHASM harmonizes channel directions without forcing all frequencies to
share identical responses.

CHASM is designed as a drop-in spectral mixer for controlled same-backbone
evaluation. This modular setting isolates the effect of cross-frequency
harmonization from changes to the surrounding architecture. We also provide a
structural characterization of the direct shared-basis operator family,
clarifying the local degree-of-freedom structure of the harmonized spectral
operator.

We evaluate CHASM on accelerated MRI reconstruction, undersampled MRI
segmentation, and natural-image reconstruction. Across CC359, fastMRI, BraTS,
ADE20K, and ImageNet settings, CHASM improves over same-backbone spectral-mixer baselines when inserted into HiFi-Mamba-family, U-Net-, and
ViT-based backbones~\citep{chen2025hifimamba,fang2025hifimambav2,ronneberger2015unet,dosovitskiy2021image}.
Mechanistic ablations further support the proposed design: releasing the
shared-basis constraint weakens performance, and disrupting coherent sampling
geometry with random masks substantially reduces the gain.

\noindent\textbf{Contributions.}
Our main contributions are threefold. 
\begin{itemize}
\item First, we identify cross-frequency harmonization as a useful design principle
for spectral token mixers, addressing the lack of explicit structure tying
frequency-indexed channel operators together. 

\item Second, we propose CHASM, an axis-separable shared-basis mixer whose spectral
core shares learned channel directions across frequencies while retaining
frequency-specific positive gains. 

\item Third, we provide a structural characterization of the corresponding
shared-basis operator family and validate CHASM through controlled
same-backbone comparisons across reconstruction and segmentation benchmarks,
together with ablations that isolate the shared-basis constraint and structured
sampling geometry.
\end{itemize}

\section{Related Work}
\label{sec:related}

\noindent\textbf{FFT-based spectral token mixing.}
Fourier-domain operations provide global mixing beyond local spatial
neighborhoods. Fourier Neural Operators parameterize global convolutional
operators in frequency space~\citep{li2021fourier}, while Fast Fourier
Convolution introduces spectral branches into convolutional networks~\citep{chi2020fast}.
In vision architectures, Global Filter Networks learn frequency-domain filters
for long-range spatial dependencies~\citep{rao2021global}, and Adaptive Fourier
Neural Operators introduce adaptive spectral mixing over selected modes
~\citep{guibas2022adaptive}. These methods demonstrate the effectiveness of spectral representations, but they do not explicitly learn a shared channel coordinate system for
frequency-indexed channel operators. CHASM studies this complementary design
axis by sharing channel directions across frequencies while retaining
frequency-specific spectral gains. Rather than increasing per-frequency flexibility, CHASM restricts the family of
frequency-indexed channel operators to share a learned eigenspace, making
cross-frequency coordination an explicit architectural constraint.

\noindent\textbf{Structured token mixers and replacement operators.}
Recent architectures increasingly separate the token mixer from the surrounding
network design. MLP-Mixer shows that token mixing can be implemented without
convolution or attention~\citep{tolstikhin2021mlp}, and MetaFormer abstracts a
common block structure in which different token mixers can be exchanged
~\citep{yu2022metaformer}. Following this view, CHASM is evaluated as a
drop-in mixer block inside fixed HiFi-Mamba-family, U-Net, and ViT backbones,
rather than as a new end-to-end architecture
~\citep{chen2025hifimamba,fang2025hifimambav2,ronneberger2015unet,dosovitskiy2021image}.

\noindent\textbf{Spectral structure in MRI reconstruction.}
Accelerated MRI reconstruction is naturally tied to frequency-domain structure:
measurements are acquired in undersampled \(k\)-space, and reconstruction must
recover missing spectral information. Classical compressed-sensing MRI exploits
sparsity priors for rapid imaging~\citep{lustig2007sparse}, and fastMRI provides
large-scale benchmarks for learning-based reconstruction~\citep{zbontar2018fastmri}.
Recent MRI backbones further emphasize frequency-aware decomposition and
high-fidelity detail recovery~\citep{chen2025hifimamba,fang2025hifimambav2}.
CHASM is orthogonal to these backbone-level designs, targeting the spectral
mixer itself through shared-basis cross-frequency harmonization.

\begin{figure*}[t]
\centering
\includegraphics[width=\textwidth]{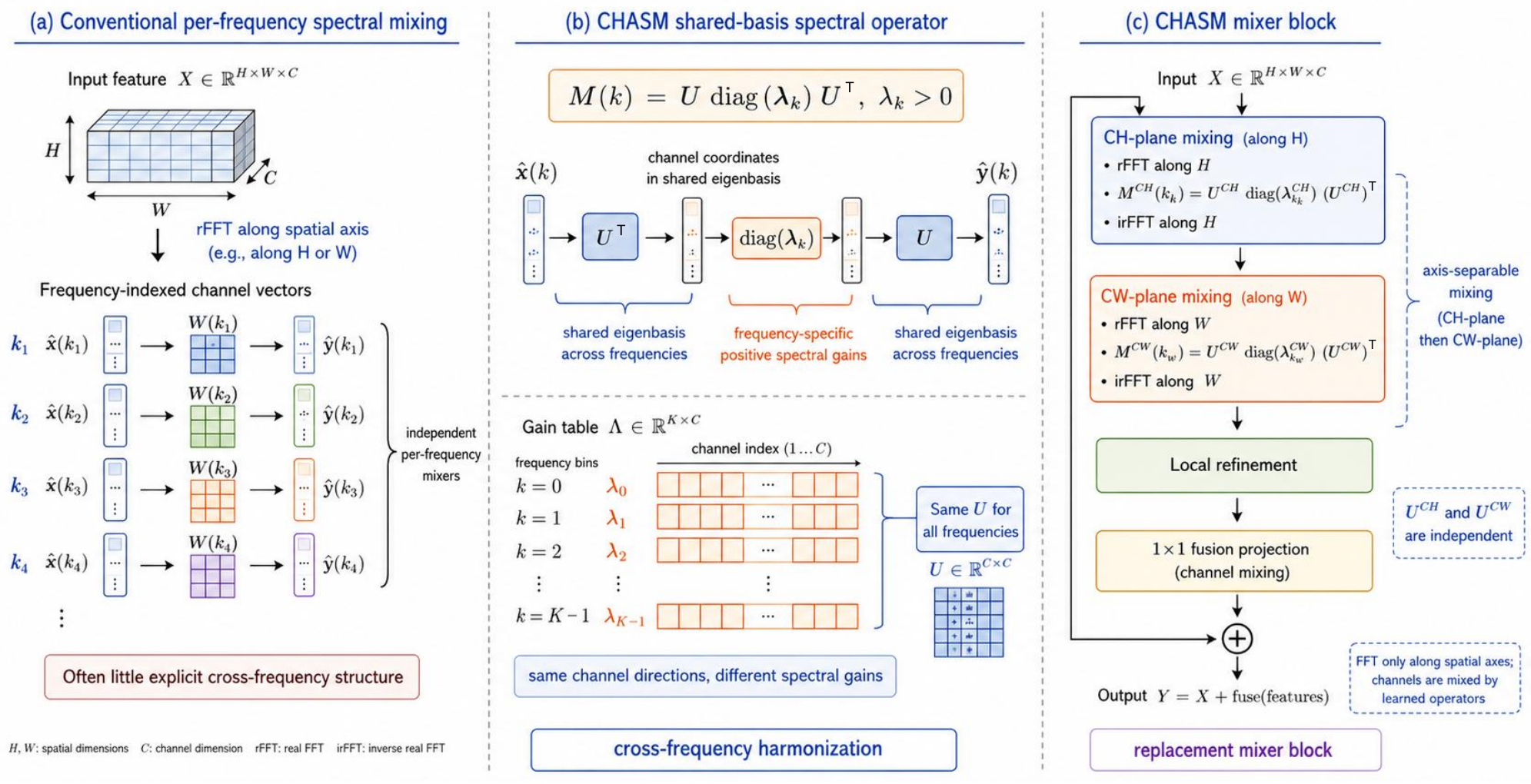}
\caption{
Overview of CHASM.
(a) Standard spectral mixers often learn independent channel operators for
different Fourier bins, with little explicit cross-frequency structure.
(b) CHASM introduces cross-frequency harmonization by sharing a learned channel
eigenbasis \(U\) across frequencies while retaining frequency-specific positive
gains \(\lambda_k\).
(c) The resulting operator is applied separably along the height--channel and
width--channel planes, followed by local refinement and \(1\times1\) fusion,
yielding a drop-in replacement mixer block.
}
\label{fig:chasm_overview}
\end{figure*}

\section{Method}
\label{sec:method}

\subsection{Real-valued representation and axis-wise spectral mixing}
\label{sec:method_prelim}

Let $\mathbf{X}\in\mathbb{R}^{H\times W\times C}$ denote a feature
tensor with height $H$, width $W$, and channel dimension $C$.
CHASM operates on real-valued network features. For complex-valued MRI
inputs, real and imaginary components are stacked along the channel
dimension before being processed by the network; thus $C$ denotes the
resulting real feature-channel dimension.

CHASM uses one-dimensional real Fourier transforms along spatial axes.
For the height--channel plane, abbreviated as the CH-plane, we write
\begin{equation}
    \widehat{\mathbf{X}}^{\mathrm{CH}}[k_h,w,:]
    =
    \mathcal{F}_{H}\bigl(\mathbf{X}[:,w,:]\bigr)[k_h],
    \qquad
    k_h=0,\ldots,K_H-1,
    \label{eq:ch_fft}
\end{equation}
where $\mathcal{F}_{H}$ denotes the real FFT along the height axis and
$K_H=\lfloor H/2\rfloor+1$ is the number of retained frequency bins.
Analogously, the width--channel plane, or CW-plane, gives
$\widehat{\mathbf{X}}^{\mathrm{CW}}[h,k_w,:]$ with
$K_W=\lfloor W/2\rfloor+1$.

At each retained spatial frequency, the mixer applies a channel operator
to the corresponding $C$-dimensional Fourier coefficient:
\begin{equation}
    \widehat{\mathbf{Z}}^{\mathrm{CH}}[k_h,w,:]
    =
    \mathbf{M}^{\mathrm{CH}}(k_h)
    \widehat{\mathbf{X}}^{\mathrm{CH}}[k_h,w,:].
\end{equation}
Although the Fourier coefficient is complex-valued,
\(\mathbf{M}^{\mathrm{CH}}(k_h)\in\mathbb{R}^{C\times C}\) is a real
channel-mixing matrix and is applied by complex linear extension. Equivalently,
the same real matrix acts on the real and imaginary parts of each Fourier
coefficient. This preserves the conjugate symmetry required by the inverse real
Fourier transform. In particular, the DC bin and, when present, the Nyquist bin
remain real-valued because they are multiplied by real matrices. No Fourier
transform is taken along the channel dimension; channels are mixed by learned
linear operators indexed by spatial frequency.

\subsection{Cross-frequency harmonized spectral operator}
\label{sec:chasm_operator}

The central design question is how the channel operators \(\{M(k)\}\) should
vary across frequencies. A frequency-indexed channel operator determines two
aspects of the spectral response: the channel directions being mixed and the
gain applied along those directions. Filter-wise designs keep the channel axes
fixed, whereas fully independent operators allow both directions and gains to
vary with frequency. CHASM separates these roles: channel directions are shared
across frequencies, while spectral gains remain frequency-specific.

For the CH-plane, the harmonized spectral operator is
\begin{equation}
    \mathbf{M}^{\mathrm{CH}}(k_h)
    =
    \mathbf{U}^{\mathrm{CH}}
    \operatorname{diag}\!\bigl(\boldsymbol{\lambda}^{\mathrm{CH}}_{k_h}\bigr)
    \bigl(\mathbf{U}^{\mathrm{CH}}\bigr)^{\top},
    \label{eq:chasm_ch_operator}
\end{equation}
where $\mathbf{U}^{\mathrm{CH}}\in\mathbb{R}^{C\times C}$ is a learned
orthogonal basis shared across all height frequencies, and
$\boldsymbol{\lambda}^{\mathrm{CH}}_{k_h}\in\mathbb{R}_{>0}^{C}$ is a
frequency-specific positive gain vector. The CW-plane uses an independent
basis and response table:
\begin{equation}
    \mathbf{M}^{\mathrm{CW}}(k_w)
    =
    \mathbf{U}^{\mathrm{CW}}
    \operatorname{diag}\!\bigl(\boldsymbol{\lambda}^{\mathrm{CW}}_{k_w}\bigr)
    \bigl(\mathbf{U}^{\mathrm{CW}}\bigr)^{\top}.
    \label{eq:chasm_cw_operator}
\end{equation}
Thus, $\mathbf{U}^{\mathrm{CH}}$ and $\mathbf{U}^{\mathrm{CW}}$ are
separate parameters. The two axes share the same structural principle,
but not the same learned channel basis.

For each axis $a\in\{\mathrm{CH},\mathrm{CW}\}$, the orthogonal basis is
parameterized with exactly $C(C-1)/2$ trainable scalars. Let
$\boldsymbol{\theta}^{a}\in\mathbb{R}^{C(C-1)/2}$ denote the parameters
assigned to the strictly lower triangular entries of a skew-symmetric
matrix $\mathbf{A}^{a}$. We construct
\begin{equation}
    A^{a}_{ij}
    =
    \begin{cases}
    \theta^{a}_{ij}, & i>j,\\
    -\theta^{a}_{ji}, & i<j,\\
    0, & i=j,
    \end{cases}
    \qquad
    \mathbf{U}^{a}
    =
    \exp(\mathbf{A}^{a}).
    \label{eq:matrix_exp}
\end{equation}
The matrix exponential maps skew-symmetric matrices to
$\mathrm{SO}(C)$. The operator family is naturally described using
orthogonal bases in $\mathrm{O}(C)$, but using an $\mathrm{SO}(C)$
representative does not reduce expressivity for the symmetric operators
considered here: if an eigenbasis has determinant $-1$, flipping the
sign of one eigenvector gives a determinant-$1$ basis and leaves
$\mathbf{U}\operatorname{diag}(\boldsymbol{\lambda})\mathbf{U}^{\top}$
unchanged. The exponential coordinates are not globally one-to-one, but
they provide the correct local degrees of freedom for the shared-basis
operator family. We initialize the skew-symmetric parameters to zero, so the learned basis starts
from the identity matrix.

The frequency-specific gains are represented by learnable tables. For
axis $a$, we learn
$\boldsymbol{\Gamma}^{a}\in\mathbb{R}^{B_a\times C}$, where $B_a$ is the
number of learned frequency bins. At a given input resolution, the table
is linearly interpolated along the normalized rFFT frequency coordinate
to the retained frequency length $K_a$, and then mapped to positive gains:
\begin{equation}
    \boldsymbol{\lambda}^{a}_{k}
    =
    \operatorname{softplus}
    \left(
    \operatorname{Interp}
    \bigl(
    \boldsymbol{\Gamma}^{a}
    \bigr)_{k,:}
    \right),
    \qquad
    k=0,\ldots,K_a-1.
    \label{eq:lambda_interp}
\end{equation}
When $B_a=K_a$, this reduces to a direct per-frequency gain table.
The response tables are initialized so that the spectral core starts near
an identity gain. We do not apply per-channel mean normalization.

For a complex Fourier coefficient $\mathbf{v}\in\mathbb{C}^{C}$, the
operator is applied as
\begin{equation}
    \mathbf{M}^{a}(k)\mathbf{v}
    =
    \mathbf{U}^{a}
    \left(
    \boldsymbol{\lambda}^{a}_{k}
    \odot
    \bigl((\mathbf{U}^{a})^{\top}\mathbf{v}\bigr)
    \right),
    \label{eq:operator_application}
\end{equation}
where $\odot$ denotes elementwise multiplication. Since
$\mathbf{U}^{a}$ and $\boldsymbol{\lambda}^{a}_{k}$ are real, the same
operation is applied consistently to the real and imaginary parts. At the
level of the spectral core, each retained-frequency channel operator is
real symmetric with a positive spectrum.

\subsection{CHASM mixer block}
\label{sec:chasm_module}

The proposed operator is used as the spectral core of a complete mixer block.
CHASM applies the harmonized operator separably along the two spatial axes,
factorizing spectral mixing into two one-dimensional frequency-indexed channel
operations over the CH and CW planes. This reduces the parameterization required
for frequency-indexed channel mixing while still enabling global interaction
along both spatial axes through composition.

Given $\mathbf{X}\in\mathbb{R}^{H\times W\times C}$, the CH-plane pass
first computes an rFFT along the height axis, applies
\eqref{eq:chasm_ch_operator} at each height frequency, and returns to the
spatial domain:
\begin{align}
    \widehat{\mathbf{X}}^{\mathrm{CH}}[k_h,w,:]
    &=
    \mathcal{F}_{H}\bigl(\mathbf{X}[:,w,:]\bigr)[k_h],
    \\
    \widehat{\mathbf{Z}}^{\mathrm{CH}}[k_h,w,:]
    &=
    \mathbf{M}^{\mathrm{CH}}(k_h)
    \widehat{\mathbf{X}}^{\mathrm{CH}}[k_h,w,:],
    \\
    \mathbf{Z}[:,w,:]
    &=
    \mathcal{F}_{H}^{-1}
    \bigl(\widehat{\mathbf{Z}}^{\mathrm{CH}}[:,w,:]\bigr).
\end{align}
The CW-plane pass then applies the analogous operation to $\mathbf{Z}$:
\begin{align}
    \widehat{\mathbf{Z}}^{\mathrm{CW}}[h,k_w,:]
    &=
    \mathcal{F}_{W}\bigl(\mathbf{Z}[h,:,:]\bigr)[k_w],
    \\
    \widehat{\mathbf{T}}^{\mathrm{CW}}[h,k_w,:]
    &=
    \mathbf{M}^{\mathrm{CW}}(k_w)
    \widehat{\mathbf{Z}}^{\mathrm{CW}}[h,k_w,:],
    \\
    \mathbf{T}[h,:,:]
    &=
    \mathcal{F}_{W}^{-1}
    \bigl(\widehat{\mathbf{T}}^{\mathrm{CW}}[h,:,:]\bigr).
\end{align}
Thus the default CHASM spectral core is
\begin{equation}
    \Phi_{\mathrm{CHASM}}
    =
    \Phi_{\mathrm{CW}}\circ \Phi_{\mathrm{CH}},
    \label{eq:chasm_composition}
\end{equation}
where the height-axis harmonized mixing is applied before the width-axis
harmonized mixing.

The full CHASM mixer block wraps this spectral core with lightweight
local refinement and channel fusion:
\begin{equation}
    \mathbf{Y}
    =
    \mathbf{X}
    +
    \mathbf{P}_{\mathrm{fuse}}
    \left(
    \mathcal{L}
    \bigl(
    \Phi_{\mathrm{CHASM}}(\mathbf{X})
    \bigr)
    \right),
    \label{eq:chasm_block}
\end{equation}
where $\mathcal{L}$ is a lightweight local convolutional refinement and
$\mathbf{P}_{\mathrm{fuse}}$ is a $1\times1$ fusion projection. The
proposed cross-frequency harmonization is contained in
$\Phi_{\mathrm{CHASM}}$; the local refinement and fusion layers make the operator a practical drop-in
mixer block for existing vision and reconstruction backbones.

Because the two axes use independent bases and frequency responses, their
composition is not merely a duplicated filtering step; the second pass acts on
features already transformed by the first, allowing axis-wise spectral mixing to
interact through the channel space.



\subsection{Structural characterization of the shared-basis operator}
\label{sec:operator_characterization}

We now characterize the structured spectral operator instantiated by one
axis of CHASM. The degree-of-freedom characterization concerns the direct-table version of the operator, in which each retained frequency has its own gain vector.
For $K$ frequency bins and channel dimension $C$, define the positive-gain
shared-basis family
\begin{equation}
    \mathcal{S}_{K,+}^{C}
    =
    \left\{
    \bigl(\mathbf{M}(0),\ldots,\mathbf{M}(K-1)\bigr)
    :
    \mathbf{M}(k)
    =
    \mathbf{U}
    \operatorname{diag}(\boldsymbol{\lambda}_{k})
    \mathbf{U}^{\top},
    \;
    \mathbf{U}\in\mathrm{O}(C),
    \;
    \boldsymbol{\lambda}_{k}\in\mathbb{R}_{>0}^{C}
    \right\}.
    \label{eq:sd_positive_family}
\end{equation}
This is the family of real symmetric channel operators with positive
spectra that are simultaneously diagonalizable by a shared orthogonal
basis. The theoretical characterization concerns this spectral operator
family, not the entire mixer block with local refinement and fusion.

For eigenmode $j$, define its joint spectral signature as
\begin{equation}
    \boldsymbol{s}_{j}
    =
    \bigl(
    \lambda_{0j},
    \lambda_{1j},
    \ldots,
    \lambda_{K-1,j}
    \bigr)
    \in\mathbb{R}_{>0}^{K}.
    \label{eq:joint_signature}
\end{equation}
The generic stratum of $\mathcal{S}_{K,+}^{C}$ is the subset where
$\boldsymbol{s}_{i}\neq \boldsymbol{s}_{j}$ for all $i\neq j$. On this
stratum, the joint eigenspaces are one-dimensional.

\begin{proposition}[Local degrees of freedom of the direct shared-basis operator]
\label{prop:generic_minimality}
On the generic stratum defined by pairwise distinct joint spectral
signatures, the direct-table family $\mathcal{S}_{K,+}^{C}$ is a smooth
manifold of intrinsic dimension
\begin{equation}
    \frac{C(C-1)}{2}+KC.
\end{equation}
The direct shared-basis spectral operator has the same number of local
degrees of freedom: $C(C-1)/2$ for the basis and $KC$ for the
frequency-specific gain table. Thus, the direct shared-basis parameterization matches the intrinsic local dimension of this generic structured family.
\end{proposition}

\begin{proof}
The orthogonal group $\mathrm{O}(C)$ has dimension $C(C-1)/2$, and the
positive gain table
$\boldsymbol{\Lambda}\in\mathbb{R}_{>0}^{K\times C}$ contributes $KC$
additional coordinates. Since $\mathbb{R}_{>0}^{K\times C}$ is an open
subset of $\mathbb{R}^{K\times C}$, the positivity constraint does not
change the local dimension. On the generic stratum, the joint spectral
signatures are pairwise distinct, so the shared eigenspaces are
one-dimensional. The remaining ambiguity is discrete: columns of
$\mathbf{U}$ may be sign-flipped, and eigenmode columns may be permuted
together with the corresponding columns of
$\boldsymbol{\Lambda}$. Quotienting by this finite signed-permutation
stabilizer does not change the manifold dimension. Thus the intrinsic
dimension of the generic structured family is
$C(C-1)/2+KC$. This verifies that the direct shared-basis parameterization matches the intrinsic local dimension of the corresponding generic structured family.
\end{proof}

This characterization focuses on the shared-basis spectral operator. The full
CHASM block further includes local refinement and fusion layers. The result
clarifies how the spectral core separates coordination from adaptivity: the
shared basis provides common channel directions across frequencies, while the
gain table preserves frequency-specific responses along those directions.

The practical implementation uses a learned response table with \(B\) frequency
bins and interpolates it to the \(K\) retained rFFT bins of the input. When
\(B<K\), this gives a smooth subparameterization of the direct-table family used
in Proposition~\ref{prop:generic_minimality}. The direct family is used for the
local degree-of-freedom characterization, while the interpolated table is used
in the implemented mixer.

\noindent\textbf{Parameter-level frequency reindexing.}
The shared-basis form has a simple parameter-level response to
frequency-index permutations. Let $\sigma$ be a permutation of
$\{0,\ldots,K-1\}$ and define
\begin{equation}
    (\sigma\cdot\boldsymbol{\Lambda})_{k,:}
    =
    \boldsymbol{\Lambda}_{\sigma^{-1}(k),:}.
\end{equation}
Then
\begin{equation}
    \mathbf{U}
    \operatorname{diag}\bigl((\sigma\cdot\boldsymbol{\Lambda})_{k,:}\bigr)
    \mathbf{U}^{\top}
    =
    \mathbf{M}(\sigma^{-1}(k)).
\end{equation}
Thus, reindexing frequencies only permutes rows of the gain table and
leaves the basis component unchanged. This property describes the parameter-level response of the shared-basis form to
frequency reindexing.

\subsection{Hermitian motivation for positive spectral gains}
\label{sec:hermitian_motivation}

The positive shared-basis form is also compatible with the operator
structure encountered in inverse problems. For a linear measurement
operator $\mathcal{A}$, the normal operator
$\mathcal{A}^{*}\mathcal{A}$ is Hermitian positive semidefinite. In
single-coil Cartesian MRI, for example,
$\mathcal{A}=\mathbf{P}_{\Omega}\mathbf{F}$ gives
\begin{equation}
    \mathcal{A}^{*}\mathcal{A}
    =
    \mathbf{F}^{*}\mathbf{P}_{\Omega}^{*}\mathbf{P}_{\Omega}\mathbf{F},
    \label{eq:single_coil_normal}
\end{equation}
where $\mathbf{F}$ is the Fourier transform and $\mathbf{P}_{\Omega}$ is
the sampling mask. In multi-coil MRI with coil sensitivities
$\{\mathbf{S}_{c}\}$, the normal operator has the form
\begin{equation}
    \mathcal{A}^{*}\mathcal{A}
    =
    \sum_{c}
    \mathbf{S}_{c}^{*}
    \mathbf{F}^{*}
    \mathbf{P}_{\Omega}^{*}\mathbf{P}_{\Omega}
    \mathbf{F}
    \mathbf{S}_{c}.
    \label{eq:multi_coil_normal}
\end{equation}
This observation motivates the positive shared-basis spectral core at the
feature level. CHASM acts on learned feature channels, where the same
Hermitian-compatible structure provides a useful inductive bias. Under complex linear extension,
$\mathbf{U}\operatorname{diag}(\boldsymbol{\lambda})\mathbf{U}^{\top}$
is Hermitian, and with $\boldsymbol{\lambda}>0$ it has a positive
spectrum at each retained frequency. This is compatible with, but not a
parameterization of, the Hermitian positive-semidefinite normal operators
that arise in inverse problems. This property belongs to the harmonized
spectral core; the full CHASM block further includes local refinement,
fusion, and residual components.

\subsection{Parameters and computational cost}
\label{sec:method_complexity}

For one axis with $K$ retained frequency bins, the direct shared-basis
spectral operator has
\begin{equation}
    {C(C-1)}/{2} + KC
\end{equation}
trainable spectral-operator parameters. In the interpolated implementation,
the learned gain table has $B$ frequency bins and is interpolated to $K$
retained bins, giving
\begin{equation}
    {C(C-1)}/{2} + BC
\end{equation}
trainable spectral-operator parameters for that axis. For the two-axis
CHASM spectral core, this becomes
\begin{equation}
    C(C-1) + (B_H+B_W)C.
    \label{eq:chasm_core_param_count}
\end{equation}
The full CHASM mixer block additionally includes the local refinement and
fusion projection in \eqref{eq:chasm_block}.

For the CH-plane pass, the FFT and inverse FFT cost
$\mathcal{O}(HWC\log H)$, and the channel mixing costs
$\mathcal{O}(WK_HC^2)$ when implemented as
$\mathbf{U}\operatorname{diag}(\boldsymbol{\lambda}_{k})\mathbf{U}^{\top}$.
For the CW-plane pass, the corresponding costs are
$\mathcal{O}(HWC\log W)$ and $\mathcal{O}(HK_WC^2)$. The total spectral
core cost is therefore
\begin{equation}
    \mathcal{O}\!\left(
    HWC(\log H+\log W)
    +
    (WK_H+HK_W)C^2
    \right),
    \label{eq:chasm_complexity}
\end{equation}
plus the cost of constructing the two orthogonal bases. The matrix
exponentials add an $\mathcal{O}(C^3)$ basis-construction cost per
forward pass. This basis-construction cost is independent of the number
of spatial positions and is amortized over all frequency bins and
locations.

\section{Experiments}
\label{sec:experiments}

We evaluate CHASM under a controlled mixer-replacement protocol across
cross-backbone MRI reconstruction, cross-domain transfer, and mechanistic
ablations. All methods replace the same mixer location inside a fixed backbone,
and each baseline is evaluated as its native complete mixer block with intrinsic
components retained. The ablations isolate the effects of the shared-basis
constraint and coherent sampling geometry by keeping the CHASM wrapper and
training protocol fixed while changing only the spectral-core constraint or the
sampling mask. Additional stability results, ablation definitions, computational
cost, reproducibility details, cross-domain protocols, and baseline taxonomy are
provided in Appendix~\ref{app:additional}.

\noindent\textbf{Protocol and metrics.}
We evaluate HiFi-Mamba-family, U-Net-based, and ViT-based backbones. For
reconstruction, we report PSNR and SSIM; for segmentation, we report mIoU and
Dice. Main test-set tables report case-level variation, while run-to-run
stability over random seeds is reported separately in
Appendix~\ref{app:seed_stability}.

\begin{table*}[h]
\centering
\footnotesize
\setlength{\tabcolsep}{2.4pt}
\renewcommand{\arraystretch}{0.95}

\providecommand{\mstd}[2]{#1\raisebox{-0.35ex}{\tiny$\pm#2$}}
\providecommand{\bmstd}[2]{\textbf{#1}\raisebox{-0.35ex}{\tiny$\pm#2$}}

\caption{
CC359 accelerated MRI reconstruction under fourfold and eightfold
undersampling. All methods replace the same mixer block. Core Params are
reported at channel dimension \(C=96\); model-level FLOPs are provided in
Appendix~\ref{app:efficiency}. Values are reported as mean with case-level standard deviation shown in smaller text.
}
\label{tab:cc359_recon}
\begin{tabular}{c l c cc cc cc}
\toprule
\multirow{2}{*}{AR} & \multirow{2}{*}{Method}
& \multirow{2}{*}{Core Params}
& \multicolumn{2}{c}{HiFi-Mamba}
& \multicolumn{2}{c}{U-Net}
& \multicolumn{2}{c}{ViT} \\
\cmidrule(lr){4-5}
\cmidrule(lr){6-7}
\cmidrule(lr){8-9}
&
&
& PSNR$\uparrow$ & SSIM$\uparrow$
& PSNR$\uparrow$ & SSIM$\uparrow$
& PSNR$\uparrow$ & SSIM$\uparrow$ \\
\midrule

\multirow{9}{*}{$\times8$}
& Baseline
& 0.00M
& \mstd{26.82}{0.72} & \mstd{0.791}{0.015}
& \mstd{24.79}{0.57} & \mstd{0.737}{0.014}
& \mstd{24.16}{0.55} & \mstd{0.722}{0.014} \\

& GFNet
& 0.80M
& \mstd{28.13}{0.75} & \mstd{0.826}{0.015}
& \mstd{25.38}{0.60} & \mstd{0.751}{0.014}
& \mstd{24.70}{0.57} & \mstd{0.733}{0.014} \\

& AFNO
& 0.01M
& \mstd{27.30}{0.73} & \mstd{0.805}{0.015}
& \mstd{25.28}{0.59} & \mstd{0.748}{0.014}
& \mstd{24.66}{0.57} & \mstd{0.732}{0.014} \\

& FFC
& 1.19M
& \mstd{27.32}{0.73} & \mstd{0.805}{0.015}
& \mstd{25.13}{0.59} & \mstd{0.744}{0.014}
& \mstd{24.50}{0.56} & \mstd{0.729}{0.014} \\

& FNO-D
& 4.74M
& \mstd{26.98}{0.72} & \mstd{0.796}{0.015}
& \mstd{25.05}{0.59} & \mstd{0.742}{0.014}
& \mstd{24.62}{0.57} & \mstd{0.731}{0.014} \\

& DiffFNO
& 4.77M
& \mstd{27.92}{0.74} & \mstd{0.819}{0.015}
& \mstd{25.31}{0.60} & \mstd{0.749}{0.014}
& \mstd{24.58}{0.57} & \mstd{0.730}{0.014} \\

& FSEL
& 0.06M
& \mstd{28.16}{0.75} & \mstd{0.827}{0.015}
& \mstd{25.42}{0.60} & \mstd{0.752}{0.014}
& \mstd{24.55}{0.57} & \mstd{0.730}{0.014} \\

& SpectFormer
& 0.02M
& \mstd{28.10}{0.75} & \mstd{0.825}{0.015}
& \mstd{25.36}{0.60} & \mstd{0.751}{0.014}
& \mstd{24.78}{0.58} & \mstd{0.735}{0.014} \\

& \textbf{CHASM}
& 0.07M
& \bmstd{28.68}{0.74} & \bmstd{0.840}{0.016}
& \bmstd{25.72}{0.61} & \bmstd{0.761}{0.014}
& \bmstd{25.19}{0.59} & \bmstd{0.745}{0.014} \\

\midrule

\multirow{9}{*}{$\times4$}
& Baseline
& 0.00M
& \mstd{34.53}{0.75} & \mstd{0.936}{0.009}
& \mstd{28.81}{0.60} & \mstd{0.852}{0.010}
& \mstd{27.86}{0.58} & \mstd{0.830}{0.010} \\

& GFNet
& 0.80M
& \mstd{35.70}{0.75} & \mstd{0.947}{0.009}
& \mstd{29.55}{0.62} & \mstd{0.866}{0.010}
& \mstd{28.62}{0.59} & \mstd{0.844}{0.010} \\

& AFNO
& 0.01M
& \mstd{35.61}{0.75} & \mstd{0.941}{0.009}
& \mstd{29.31}{0.61} & \mstd{0.860}{0.010}
& \mstd{28.36}{0.59} & \mstd{0.838}{0.010} \\

& FFC
& 1.19M
& \mstd{35.62}{0.75} & \mstd{0.946}{0.009}
& \mstd{29.24}{0.61} & \mstd{0.858}{0.010}
& \mstd{28.45}{0.59} & \mstd{0.840}{0.010} \\

& FNO-D
& 4.74M
& \mstd{34.77}{0.75} & \mstd{0.939}{0.009}
& \mstd{29.20}{0.61} & \mstd{0.858}{0.010}
& \mstd{28.34}{0.59} & \mstd{0.838}{0.010} \\

& DiffFNO
& 4.77M
& \mstd{34.13}{0.74} & \mstd{0.933}{0.009}
& \mstd{29.28}{0.61} & \mstd{0.859}{0.010}
& \mstd{28.28}{0.59} & \mstd{0.836}{0.010} \\

& FSEL
& 0.06M
& \mstd{36.22}{0.76} & \mstd{0.951}{0.008}
& \mstd{29.61}{0.62} & \mstd{0.867}{0.010}
& \mstd{28.56}{0.59} & \mstd{0.843}{0.010} \\

& SpectFormer
& 0.02M
& \mstd{35.91}{0.75} & \mstd{0.948}{0.008}
& \mstd{29.50}{0.62} & \mstd{0.865}{0.010}
& \mstd{28.65}{0.59} & \mstd{0.845}{0.010} \\

& \textbf{CHASM}
& 0.07M
& \bmstd{37.05}{0.76} & \bmstd{0.956}{0.007}
& \bmstd{30.03}{0.63} & \bmstd{0.875}{0.009}
& \bmstd{29.10}{0.60} & \bmstd{0.856}{0.009} \\

\bottomrule
\end{tabular}
\end{table*}

\subsection{Accelerated MRI reconstruction}
\label{sec:mri_reconstruction}

We first evaluate accelerated MRI reconstruction on CC359~\citep{souza2018open}
and fastMRI~\citep{zbontar2018fastmri}. All methods share the same setting. Table~\ref{tab:cc359_recon} reports CC359 results across three backbone
families. CHASM improves over same-backbone spectral-mixer blocks
under both $4\times$ and $8\times$ acceleration, indicating that the
shared-basis spectral operator is not tied to a single backbone. The gain persists across recurrent/state-space-style, convolutional, and
Transformer-style hosts, supporting the use of CHASM as a replacement spectral
mixer rather than a backbone-specific modification.

\subsection{Additional reconstruction and segmentation benchmarks}
\label{sec:additional_benchmarks}
\noindent\textbf{Role of the cross-domain evaluation.}
The additional tasks in Table~\ref{tab:unified_cross_domain_results} are used
as controlled transfer tests rather than task-specific state-of-the-art
benchmarks. In each setting, the backbone, degradation or undersampling
operator, objective, and training protocol are fixed across mixer blocks, so
the comparison isolates whether the shared-basis spectral core remains
beneficial outside the primary accelerated-MRI reconstruction setting. Detailed
protocol settings are summarized in Appendix~\ref{app:reproducibility}.

To test whether CHASM is specific to a single MRI dataset, we further
evaluate fastMRI Knee reconstruction with the HiFi-Mamba-family backbone.
We also evaluate natural-image reconstruction on ADE20K~\citep{zhou2017scene}
and ImageNet~\citep{deng2009imagenet}, and undersampled brain tumor
segmentation on BraTS~\citep{menze2015multimodal}. For natural images,
all methods use the same downsampling degradation and reconstruction
objective. For BraTS, inputs are retrospectively undersampled and models
predict tumor segmentation masks.

\begin{table*}[t]
\centering
\footnotesize
\setlength{\tabcolsep}{3pt}
\renewcommand{\arraystretch}{1.1}
\caption{
Unified evaluation across fastMRI reconstruction, natural-image
reconstruction, and undersampled MRI segmentation. The last row reports the maximum case-level standard deviation among all methods
under each setting.
}
\label{tab:unified_cross_domain_results}
\begin{tabular}{l|cc|cc|cc|cc|cc|cc}
\toprule
\multirow{2}{*}{Mixer block}
& \multicolumn{2}{c|}{fastMRI $\times8$}
& \multicolumn{2}{c|}{fastMRI $\times4$}
& \multicolumn{2}{c|}{ADE20K}
& \multicolumn{2}{c|}{ImageNet}
& \multicolumn{2}{c|}{BraTS $\times4$}
& \multicolumn{2}{c}{BraTS $\times8$} \\
\cmidrule(lr){2-3}
\cmidrule(lr){4-5}
\cmidrule(lr){6-7}
\cmidrule(lr){8-9}
\cmidrule(lr){10-11}
\cmidrule(lr){12-13}
& PSNR$\uparrow$ & SSIM$\uparrow$
& PSNR$\uparrow$ & SSIM$\uparrow$
& PSNR$\uparrow$ & SSIM$\uparrow$
& PSNR$\uparrow$ & SSIM$\uparrow$
& mIoU$\uparrow$ & Dice$\uparrow$
& mIoU$\uparrow$ & Dice$\uparrow$ \\
\midrule
Baseline
& 31.24 & 0.741
& 34.31 & 0.849
& 23.67 & 0.750
& 27.89 & 0.835
& 0.727 & 0.841
& 0.612 & 0.757 \\

GFNet
& 31.56 & 0.761
& 34.50 & 0.853
& 23.97 & 0.762
& 28.20 & 0.846
& 0.737 & 0.848
& 0.632 & 0.773 \\

AFNO
& 31.32 & 0.752
& 34.47 & 0.850
& 23.74 & 0.753
& 27.97 & 0.838
& 0.731 & 0.844
& 0.621 & 0.765 \\

FFC
& 31.40 & 0.757
& 34.51 & 0.852
& 23.73 & 0.752
& 27.95 & 0.837
& 0.730 & 0.843
& 0.619 & 0.763 \\

FNO-D
& 31.32 & 0.756
& 34.49 & 0.852
& 23.76 & 0.755
& 28.00 & 0.840
& 0.733 & 0.845
& 0.625 & 0.768 \\

DiffFNO
& 31.28 & 0.755
& 34.59 & 0.854
& 23.79 & 0.756
& 28.04 & 0.841
& 0.735 & 0.846
& 0.628 & 0.770 \\

FSEL
& 31.61 & 0.761
& 34.57 & 0.854
& 23.98 & 0.765
& 28.22 & 0.849
& 0.739 & 0.850
& 0.639 & 0.778 \\

SpectFormer
& 31.76 & 0.764
& 34.65 & 0.855
& 24.01 & 0.766
& 28.25 & 0.850
& 0.742 & 0.853
& 0.643 & 0.781 \\

\midrule
\textbf{CHASM}
& \textbf{32.31} & \textbf{0.779}
& \textbf{35.22} & \textbf{0.872}
& \textbf{24.52} & \textbf{0.783}
& \textbf{28.58} & \textbf{0.864}
& \textbf{0.792} & \textbf{0.870}
& \textbf{0.699} & \textbf{0.811} \\

\emph{Max std.}
& 1.85 & 0.076
& 2.24 & 0.052
& 1.80 & 0.045
& 1.65 & 0.040
& 0.035 & 0.030
& 0.050 & 0.045 \\
\bottomrule
\end{tabular}
\end{table*}

\begin{figure*}[t]
    \centering
    \includegraphics[width=\linewidth]{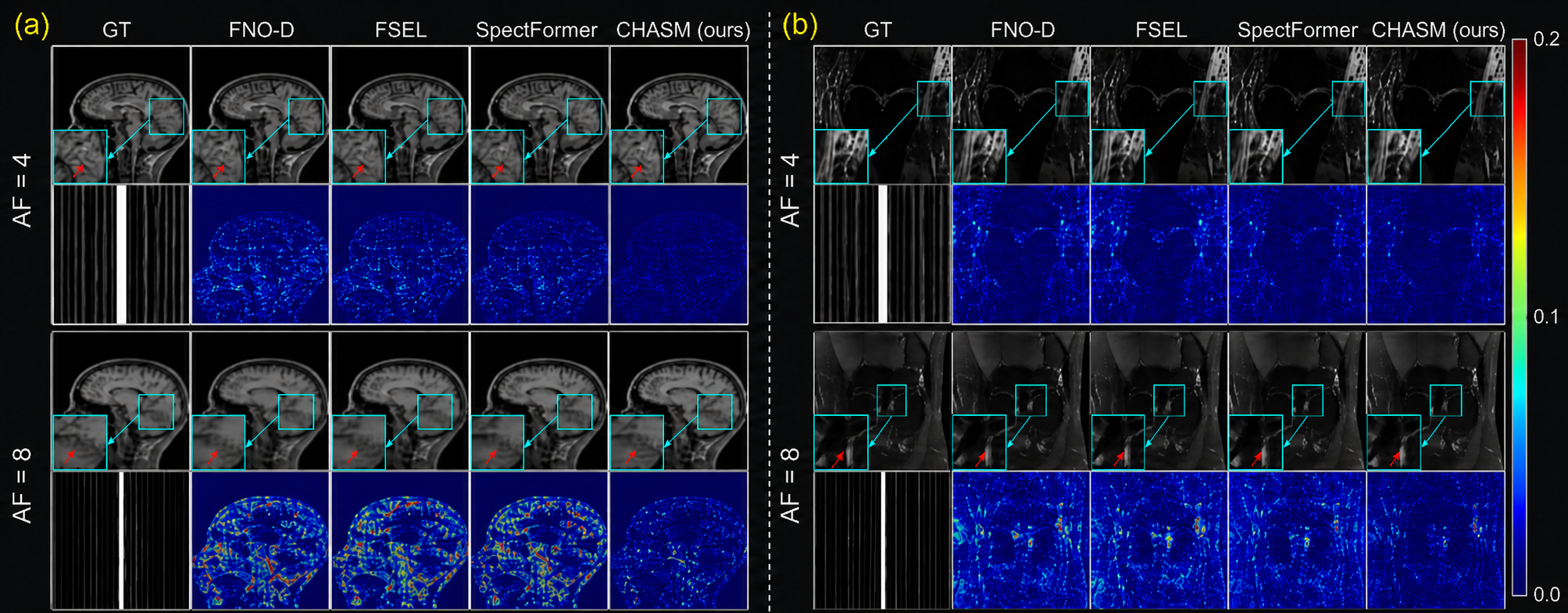}
    \caption{Qualitative comparison on the fastMRI and CC359 datasets under single-coil settings. (a) Reconstruction results on the fastMRI knee dataset with acceleration factors AF=4 and AF=8. (b) Reconstruction results on the CC359 brain dataset under the same acceleration factors. The second row of each subplot shows the corresponding error maps. The blue boxes and red arrow highlight the details in the reconstruction results.}
    \label{fig:qualitative}
\end{figure*}

Table~\ref{tab:unified_cross_domain_results} shows that CHASM improves
fastMRI reconstruction, natural-image reconstruction, and undersampled
MRI segmentation, suggesting that cross-frequency harmonization is useful
beyond a single dataset or task. Figure~\ref{fig:qualitative} further provides qualitative reconstruction examples on fastMRI and CC359. The error maps show that CHASM reduces residual
aliasing and better preserves fine anatomical structures under both fourfold and
eightfold acceleration.

\subsection{Mechanistic analysis}
\label{sec:mechanistic}

We next test whether the gain comes from the proposed cross-frequency
harmonization. Unless otherwise stated, ablations are conducted on CC359
\(\times8\) reconstruction with the HiFi-Mamba-family backbone. The
shared-basis ablation keeps the CHASM wrapper, local refinement, fusion
projection, and training protocol fixed while changing only the spectral core;
the random-mask study changes the sampling geometry and measures how the CHASM
gain varies with spectral structure. Definitions of the ablation variants and
three-seed stability for the shared-basis ablation are provided in
Appendix~\ref{app:ablation_definitions} and
Appendix~\ref{app:shared_basis_seed}.

Table~\ref{tab:chasm_random_ablation_compact}(a) shows that CHASM has a
large advantage under structured masks, but this advantage is
substantially reduced when the sampling geometry is randomized. This
supports the interpretation that cross-frequency harmonization is most
useful when the measurement process contains coherent spectral
relationships. Table~\ref{tab:chasm_random_ablation_compact}(b) shows
that removing the shared-basis constraint weakens performance even though
the rest of the CHASM block is unchanged, isolating cross-frequency
harmonization from local refinement or fusion effects. This suggests that
simply allowing frequency-specific channel directions is not sufficient; the
shared coordinate system itself contributes to the gain.
Table~\ref{tab:chasm_random_ablation_compact}(c) further shows that the
default CH$\rightarrow$CW axis-separable composition gives the strongest
performance among the tested variants.

\begin{table*}[t]
\centering
\footnotesize
\setlength{\tabcolsep}{3pt}
\renewcommand{\arraystretch}{1.1}
\caption{
Random-mask falsification and ablation studies with the HiFi-Mamba backbone.
(a) Random-mask falsification reports $\Delta$PSNR of CHASM over the
strongest spectral-mixer baseline. Structured-mask gains are computed from
the corresponding main reconstruction results. Random-mask gains are
measured by replacing structured masks with random masks under the same
model and training setting. Drop is computed as
$1{-}\Delta_{\mathrm{random}}/\Delta_{\mathrm{structured}}$.
(b) Spectral-core basis constraints are evaluated on CC359 $\times8$.
(c) CH/CW axis compositions are evaluated on CC359 $\times8$.
}
\label{tab:chasm_random_ablation_compact}

\begin{minipage}[t]{0.305\textwidth}
\centering
\textbf{(a) Random-mask falsification}
\vspace{0.4ex}

\begin{tabular}{@{}lccc@{}}
\toprule
Setting & Struct. & Random & Drop \\
\midrule
CC359 $\times4$   & +0.83 & +0.10 & 88.0\% \\
CC359 $\times8$   & +0.52 & +0.08 & 84.6\% \\
fastMRI $\times4$ & +0.57 & +0.08 & 86.0\% \\
fastMRI $\times8$ & +0.55 & +0.06 & 89.1\% \\
\textbf{Mean} & \textbf{+0.62} & \textbf{+0.08} & \textbf{86.9\%} \\
\bottomrule
\end{tabular}
\end{minipage}
\hspace{0.18in}
\begin{minipage}[t]{0.325\textwidth}
\centering
\textbf{(b) Shared-basis constraint}
\vspace{0.4ex}

\begin{tabular}{@{}lcc@{}}
\toprule
Variant & PSNR$\uparrow$ & SSIM$\uparrow$ \\
\midrule
Identity basis            & 27.86 & 0.820 \\
Untied basis              & 28.43 & 0.833 \\
Shared basis, signed gain & 28.51 & 0.836 \\
Shared basis, complex     & 28.56 & 0.838 \\
\textbf{CHASM}            & \textbf{28.68} & \textbf{0.840} \\
\bottomrule
\end{tabular}
\end{minipage}
\hspace{0.07in}
\begin{minipage}[t]{0.305\textwidth}
\centering
\textbf{(c) Axis factorization}
\vspace{0.4ex}

\begin{tabular}{@{}lcc@{}}
\toprule
Variant & PSNR$\uparrow$ & SSIM$\uparrow$ \\
\midrule
CH only             & 28.31 & 0.831 \\
CW only             & 28.27 & 0.829 \\
CW $\rightarrow$ CH & 28.54 & 0.837 \\
CH $+$ CW           & 28.48 & 0.835 \\
\textbf{CH $\rightarrow$ CW} & \textbf{28.68} & \textbf{0.840} \\
\bottomrule
\end{tabular}
\end{minipage}

\end{table*}

\section{Conclusion and Limitations}
\label{sec:conclusion}

We presented CHASM, a structured spectral mixer that harmonizes channel
directions across frequencies while preserving frequency-specific gains. By
using a shared learned eigenbasis and frequency-dependent positive gains, CHASM
provides a simple way to coordinate spectral channel mixing without removing
local frequency adaptivity. Across accelerated MRI reconstruction, natural-image
reconstruction, and undersampled MRI segmentation, CHASM improves over
same-backbone spectral-mixer baselines under controlled replacement settings.
The ablations further show that both the shared-basis constraint and coherent
sampling geometry contribute to the observed gains, supporting cross-frequency
harmonization as a useful design principle for spectral token operators. This work is limited to image-domain reconstruction and segmentation. Future
work can study broader dense prediction, real sensor-noise restoration, video
reconstruction, and a deeper characterization of when cross-frequency mixing is
most effective across sampling patterns, feature resolutions, and backbones.

\bibliographystyle{plainnat}
\bibliography{references}

\appendix

\section{Additional Experimental Details}
\label{app:additional}

\subsection{Seed Stability}
\label{app:seed_stability}

We evaluate run-to-run stability on the main MRI reconstruction settings.
We report three-seed stability for CHASM and the strongest spectral-mixer
baseline in each main MRI reconstruction setting. This appendix reports
seed-level standard deviations, separate from the case-level standard deviations
used in the main test-set tables.

\begin{table}[h]
\centering
\footnotesize
\caption{
Three-seed stability on the main MRI reconstruction settings. CHASM is compared
against the strongest spectral-mixer baseline in each setting. Values are
reported as mean $\pm$ standard deviation.
}
\label{tab:seed_stability}
\begin{tabular}{llccccc}
\toprule
Setting & Best baseline
& Base PSNR$\uparrow$ & CHASM PSNR$\uparrow$ & $\Delta$PSNR
& Base SSIM$\uparrow$ & CHASM SSIM$\uparrow$ \\
\midrule
CC359 $\times8$ & FSEL
& 28.16$\pm$0.03 & \textbf{28.68$\pm$0.02} & \textbf{+0.52}
& 0.827$\pm$0.002 & \textbf{0.840$\pm$0.001} \\

CC359 $\times4$ & FSEL
& 36.22$\pm$0.05 & \textbf{37.05$\pm$0.03} & \textbf{+0.83}
& 0.951$\pm$0.001 & \textbf{0.956$\pm$0.001} \\

fastMRI $\times8$ & SpectFormer
& 31.76$\pm$0.04 & \textbf{32.31$\pm$0.05} & \textbf{+0.55}
& 0.764$\pm$0.002 & \textbf{0.779$\pm$0.003} \\

fastMRI $\times4$ & SpectFormer
& 34.65$\pm$0.02 & \textbf{35.22$\pm$0.03} & \textbf{+0.57}
& 0.855$\pm$0.001 & \textbf{0.872$\pm$0.002} \\
\bottomrule
\end{tabular}
\end{table}

\subsection{Shared-Basis Ablation Stability}
\label{app:shared_basis_seed}

We additionally evaluate the stability of the shared-basis ablation on the
main CC359 $\times8$ reconstruction setting. The CHASM wrapper, local
refinement, fusion projection, optimizer, data split, mask, and training
schedule are kept fixed; only the spectral-core basis constraint is changed.

\begin{table}[h]
\centering
\footnotesize
\caption{
Three-seed stability for the shared-basis ablation on CC359 $\times8$ with the
HiFi-Mamba-family reconstruction backbone. Values are reported as
mean $\pm$ standard deviation.
}
\label{tab:seed_shared_basis}
\begin{tabular}{lcc}
\toprule
Variant & PSNR$\uparrow$ & SSIM$\uparrow$ \\
\midrule
Untied basis & 28.43$\pm$0.04 & 0.833$\pm$0.002 \\
CHASM        & \textbf{28.68$\pm$0.02} & \textbf{0.840$\pm$0.001} \\
\bottomrule
\end{tabular}
\end{table}

\subsection{Ablation Variant Definitions}
\label{app:ablation_definitions}

\paragraph{Basis constraints.}
The identity-basis variant fixes the shared basis to \(U=I\) and learns only
the frequency-specific gains. The untied-basis variant removes the shared-basis
constraint by learning frequency-indexed channel operators without a common
basis, while keeping the CHASM wrapper, local refinement, fusion projection,
optimizer, mask, split, and training schedule fixed. The signed-gain variant
keeps the shared basis but removes the positivity constraint on the gains,
replacing the softplus mapping with unconstrained real-valued gains. The complex
variant replaces the real-valued spectral-core gains with complex-valued
frequency-domain gains while preserving the real-output constraint of the
inverse real Fourier transform.

\paragraph{Axis compositions.}
The CH-only and CW-only variants apply the harmonized operator only along the
height--channel or width--channel plane, respectively. The CW\(\rightarrow\)CH
and CH\(\rightarrow\)CW variants apply the two axis-wise operators sequentially
in the indicated order. The CH+CW variant applies the CH-plane and CW-plane
operators in parallel and combines their outputs before the local refinement and
fusion projection.

\subsection{Computational Cost}
\label{app:efficiency}

\begin{table}[H]
\centering
\footnotesize
\caption{
Model-level FLOPs for CC359 reconstruction backbones after replacing the mixer
block. FLOPs are measured under the same input setting as Table~\ref{tab:cc359_recon}
and are reported for reference. They are deterministic for a fixed input size
and independent of random seed.
}
\label{tab:model_flops}
\begin{tabular}{lccc}
\toprule
Mixer block & HiFi-Mamba & U-Net & ViT \\
\midrule
Baseline    & 24.64G & 12.56G & 11.54G \\
GFNet       & 24.62G & 12.56G & 11.01G \\
AFNO        & 24.62G & 12.56G & 11.01G \\
FFC         & 25.57G & 12.74G & 11.13G \\
FNO-D      & 29.06G & 13.36G & 10.56G \\
DiffFNO     & 21.26G & 11.88G & 11.01G \\
FSEL        & 34.96G & 14.71G & 12.44G \\
SpectFormer & 26.43G & 12.96G & 11.27G \\
CHASM       & 28.22G & 13.27G & 11.48G \\
\bottomrule
\end{tabular}
\end{table}

\subsection{Reproducibility Details}
\label{app:reproducibility}

All comparisons use the same backbone, dataset split, undersampling mask,
training objective, optimizer, and training schedule within each experimental
setting. Each method replaces the same mixer location as a complete
non-cumulative mixer block. Components intrinsic to each baseline mixer, such
as projection layers, local convolutional branches, or frequency-selection
modules, are retained. We do not attach CHASM-specific local refinement to other
baselines. Baseline mixer configurations follow their original implementations, with only channel- and resolution-dependent dimensions adjusted to match the host backbone.

\paragraph{Datasets.}
We evaluate accelerated MRI reconstruction on CC359 and fastMRI, natural-image
reconstruction on ADE20K and ImageNet, and undersampled MRI segmentation on
BraTS. For MRI reconstruction, the same acceleration factors and undersampling
protocols are used for all compared methods. For natural-image reconstruction,
all methods use the same degradation operator and reconstruction loss. For
BraTS segmentation, inputs are retrospectively undersampled and models predict
tumor segmentation masks.

\paragraph{Cross-domain settings.}
For Table~\ref{tab:unified_cross_domain_results}, fastMRI reconstruction uses
the HiFi-Mamba-family reconstruction backbone. ADE20K and ImageNet
reconstruction use the same U-Net-style reconstruction backbone with bicubic
downsampling at \(4\times\), trained with an \(\ell_1\) reconstruction loss.
BraTS segmentation uses a U-Net segmentation backbone with retrospective
Cartesian undersampling at \(4\times\) and \(8\times\), trained with a
Dice--cross-entropy loss.

\paragraph{Metrics.}
For reconstruction, we report PSNR and SSIM. For segmentation, we report mIoU
and Dice. Higher is better for PSNR, SSIM, mIoU, and Dice.

\paragraph{Training protocol.}
All methods are trained under matched settings within each task. Unless
otherwise specified, we use AdamW with a learning rate of \(5\times10^{-4}\),
batch size \(8\), and train for \(100\) epochs. The same training and validation
splits, undersampling masks, loss function, optimizer, learning-rate setting,
batch size, training length, and data augmentation are used for all mixer blocks
within each experimental setting. Hyperparameters are fixed across methods and
test results are reported using the corresponding validation-selected checkpoint.

\paragraph{Compute.}
Experiments are run on NVIDIA A100 and H100 GPU workers. Unless otherwise
specified, all methods are trained with the same optimizer, batch size, and
training length within each experimental setting. Since all methods are evaluated as replacement mixer blocks inside the same
backbone, we report model-level FLOPs under the matched input setting in
Appendix~\ref{app:efficiency}. FLOPs are deterministic for a fixed input size
and are independent of random seed. Wall-clock
runtime can vary with FFT kernels, hardware utilization, and implementation
details, so we use FLOPs as the main hardware-independent efficiency indicator.

\paragraph{Implementation.}
CHASM uses real FFT and inverse real FFT along spatial axes, implemented with
\texttt{torch.fft.rfft} and \texttt{torch.fft.irfft}. The channel operators are
real-valued and applied to complex Fourier coefficients by complex linear
extension, equivalently applying the same real matrix to the real and imaginary
parts. The orthogonal bases are constructed from skew-symmetric matrices through
the matrix exponential, and frequency-specific gains are mapped to positive
values using softplus. The learned gain tables are interpolated to the retained rFFT frequency bins used by the input resolution.

\paragraph{Block wrapper details.}
The local refinement module \(\mathcal{L}\) is implemented as a depthwise
\(3\times3\) convolution with padding 1 followed by a GELU activation. The
fusion projection \(\mathbf{P}_{\mathrm{fuse}}\) is a \(1\times1\) convolution
that maps the refined features back to \(C\) channels. No channel expansion is
used in the wrapper. In the mechanistic ablations, \(\mathcal{L}\),
\(\mathbf{P}_{\mathrm{fuse}}\), optimizer, mask, split, and training schedule
are kept fixed; only the spectral-core constraint is changed.

\paragraph{Cross-domain protocol details.}
Table~\ref{tab:cross_domain_protocol} summarizes the protocol used for the
cross-domain transfer experiments in Table~\ref{tab:unified_cross_domain_results}. These
experiments are intended to test whether the same mixer-replacement principle
remains effective beyond the primary accelerated-MRI reconstruction setting,
rather than to establish task-specific state of the art. Within each setting,
all compared mixer blocks use the same backbone, input operator, objective,
training split, optimizer, and schedule, and replace the same mixer location.
Thus, the comparison isolates the effect of the spectral mixer under matched
cross-domain conditions.

\begin{table}[h]
\centering
\footnotesize
\caption{Protocol details for the cross-domain transfer evaluation.}
\label{tab:cross_domain_protocol}
\begin{tabular}{lcccc}
\toprule
Setting & Backbone & Input operator & Objective & Purpose \\
\midrule
fastMRI & HiFi-Mamba-family & Cartesian undersampling & reconstruction loss & MRI transfer \\
ADE20K & U-Net-style & $4\times$ bicubic downsampling & $\ell_1$ & natural-image reconstruction \\
ImageNet & U-Net-style & $4\times$ bicubic downsampling & $\ell_1$ & natural-image reconstruction \\
BraTS & U-Net segmentation & Cartesian undersampling & Dice--CE & undersampled segmentation \\
\bottomrule
\end{tabular}
\end{table}

\subsection{Taxonomy of Spectral Mixer Baselines}
\label{app:spectral_taxonomy}

Table~\ref{tab:spectral_taxonomy} summarizes how CHASM differs from existing
spectral mixers: it explicitly imposes a shared channel eigenbasis across
frequencies while preserving frequency-specific gains.

\begin{table}[h]
\centering
\footnotesize
\caption{
Taxonomy of spectral mixer designs. ``Shared basis'' means that the method
directly ties channel directions across frequency bins through a learned shared
eigenbasis.
}
\label{tab:spectral_taxonomy}
\begin{tabular}{lcccc}
\toprule
Mixer & Spectral response & Channel mixing & Shared basis & Positive gains \\
\midrule
GFNet & learned filters & diagonal/filter-wise & no & no \\
AFNO & adaptive Fourier modes & block/channel & no explicit & no \\
FFC & Fourier branch & conv./spectral & no explicit & no \\
FNO-D & learned FNO response & mode-wise & no explicit & no \\
DiffFNO & dynamic Fourier response & frequency-conditioned & no explicit & no \\
FSEL & frequency selection & learned enhancement & no explicit & no \\
SpectFormer & spectral attention & token/channel & no explicit & no \\
\textbf{CHASM} & shared-basis gains & symmetric channel & \textbf{yes} & \textbf{yes} \\
\bottomrule
\end{tabular}
\end{table}

GFNet-style mixers learn frequency-domain filters that give each Fourier bin a
learnable response, but the response is not parameterized through common channel
directions shared across the spectrum. AFNO-style mixers introduce adaptive
Fourier mixing and structured channel operations over selected modes, but the
per-frequency transformations are not constrained to be simultaneously
diagonalizable by a shared learned eigenbasis. FNO variants learn spectral
operator responses over modes, and deformation or dynamic variants can improve
adaptivity, but they still do not explicitly tie frequency-indexed channel
operators through shared eigenvectors. Spectral-attention methods can introduce
implicit coupling through attention, but this differs from the direct
shared-basis structure used by CHASM.

CHASM is therefore complementary to these designs. It keeps the useful property
of frequency-specific responses, but changes how the responses are coordinated:
all frequencies use comparable channel directions, and only their gains along
those directions vary with frequency.


\end{document}